\documentclass[11pt,a4paper]{article}
\usepackage[hyperref]{acl2020}
\usepackage{times}
\usepackage{latexsym}
\usepackage[textwidth=0.7in,textsize=tiny]{todonotes}
\usepackage{soul}

\usepackage{microtype}

\aclfinalcopy 
\def\aclpaperid{1781} 

\usepackage{amsmath,amsthm,amsfonts,amssymb}
\usepackage{bm}
\usepackage[ruled]{algorithm2e}
\usepackage{graphicx}
\usepackage{sidecap}
\usepackage{mathrsfs}
\usepackage{multirow, multicol}
\usepackage{subcaption}
\usepackage{caption}

\usepackage{wrapfig}
\usepackage{booktabs}

\usepackage{natbib}

\newtheorem{theorem}{Theorem}[section]

\newcommand{\VI}{\text{VI}}
\newcommand{\KL}{\text{KL}}

\newcommand{\MI}{\text{I}}
\newcommand{\Enpy}{\text{H}}

\newcommand{\bbE}{\mathbb{E}}

\newcommand{\calL}{\mathcal{L}}

\def\1{\bm{1}}

\def\vc{{\bm{c}}}

\def\vs{{\bm{s}}}

\def\vx{{\bm{x}}}
\def\vy{{\bm{y}}}
\def\vz{{\bm{z}}}

\DeclareMathAlphabet{\mathsfit}{\encodingdefault}{\sfdefault}{m}{sl}
\SetMathAlphabet{\mathsfit}{bold}{\encodingdefault}{\sfdefault}{bx}{n}

\title{ {I}mproving {D}isentangled {T}ext {R}epresentation {L}earning \smallskip \\
with {I}nformation-{T}heoretic {G}uidance}

\author{Pengyu Cheng$^\mathbf{1}\thanks{~~This work was conducted while the first author was doing an internship at NEC Labs America.}$, ~Martin Renqiang Min$^\mathbf{2}$, ~Dinghan Shen$^\mathbf{3}$, ~Christopher Malon$^\mathbf{2}$,\\
\textbf{Yizhe Zhang$^\mathbf{4}$, ~Yitong Li$^\mathbf{1}$, ~Lawrence Carin$^\mathbf{1}$}
 \\
 	$^{\mathbf{1}}$Duke University~~
	$^{\mathbf{2}}$NEC Labs America\\
	$^{\mathbf{3}}$Microsoft Dynamics 365 AI~~
	$^{\mathbf{4}}$Microsoft Research\\
  \texttt{ pengyu.cheng@duke.edu}}

\date{}
\begin{document}

\maketitle
\begin{abstract}
Learning disentangled 
representations of natural language is essential for many NLP tasks, \textit{e.g.}, conditional text generation, style transfer, personalized dialogue systems, \textit{etc}. Similar problems have been studied extensively for other forms of data, such as images and videos.  However, the discrete nature of natural language makes the disentangling of textual representations more challenging (\textit{e.g.}, the manipulation over the data space cannot be easily achieved). Inspired by information theory, we propose a novel method that effectively manifests disentangled representations of text, without any supervision on semantics. A new mutual information upper bound is derived and leveraged to measure dependence between style and content. By minimizing this upper bound, the proposed method induces style and content embeddings into two independent low-dimensional spaces. Experiments on both conditional text generation and text-style transfer demonstrate the high quality of our disentangled representation in terms of content and style preservation.
\end{abstract}

\section{Introduction}

 Disentangled representation learning (DRL), which maps different aspects of data into distinct and independent low-dimensional latent vector spaces, has attracted considerable attention for making deep learning models more interpretable.  Through a series of operations such as selecting, combining, and switching, the learned disentangled representations can be utilized for downstream tasks, such as domain adaptation~\citep{liu2018detach}, style transfer~\citep{lee2018diverse}, conditional generation~\citep{denton2017unsupervised,burgess2018understanding}, and few-shot learning~\citep{kumar2018generalized}. Although widely used in various domains, such as images~\citep{tran2017disentangled,lee2018diverse}, videos~\citep{yingzhen2018disentangled,hsieh2018learning}, and speech~\citep{chou2018multi,zhou2019talking}, many challenges in DRL have received limited exploration in natural language processing~\citep{john2018disentangled}.
 
 To disentangle various attributes of text, two distinct types of embeddings are typically considered: the \textit{style embedding} and the  \textit{content embedding}~\citep{john2018disentangled}. The content embedding is designed to encapsulate the semantic meaning of a sentence. In contrast, the style embedding should represent desired attributes, such as the sentiment of a review, or the personality associated with a post. Ideally, a disentangled-text-representation model should learn representative embeddings for both style and content.

To accomplish this, several strategies have been introduced. \citet{shen2017style} proposed to learn a semantically-meaningful content embedding space by matching the content embedding from two different style domains. 
However, their method requires predefined style domains, and thus cannot automatically infer style information from unlabeled text. 
\citet{hu2017toward} and \citet{lample2018multipleattribute} utilized one-hot vectors as style-related features (instead of inferring the style embeddings from the original data). These models are not applicable when new data comes from an unseen style class. \citet{john2018disentangled} proposed an encoder-decoder model in combination with an adversarial training objective to infer both style and content embeddings from the original data. However, their adversarial training framework requires manually-processed supervised information for content embeddings (\textit{e.g.}, reconstructing sentences with manually-chosen sentiment-related words removed). Further, there is no theoretical guarantee for the quality of disentanglement.

In this paper, we introduce a novel \textbf{I}nformation-theoretic \textbf{D}isentangled \textbf{E}mbedding \textbf{L}earning method (IDEL) for text, based on guidance from information theory. 
 Inspired by Variation of Information (VI), we introduce a novel information-theoretic objective to measure how well the learned representations are disentangled. Specifically, our IDEL reduces the dependency between style and content embeddings by minimizing a sample-based mutual information upper bound. Furthermore, the mutual information between latent embeddings and the input data is also maximized to ensure the representativeness of the latent embeddings (\textit{i.e.}, style and content embeddings).
The contributions of this paper are summarized as follows:
\vspace{-1.mm}
\begin{itemize}

   \item A principled framework is introduced to learn disentangled representations of natural language. By minimizing a novel VI-based DRL objective, our model not only explicitly reduces the correlation between style and content embeddings, but also simultaneously preserves the sentence information in the latent spaces.
  
\vspace{-1.mm}
   \item A general sample-based mutual information upper bound is derived to facilitate the minimization of our VI-based objective. With this new upper bound, the dependency of style and content embeddings can be decreased effectively and stably.
   
\vspace{-1.mm}
    
    \item The proposed model is evaluated empirically relative to other disentangled representation learning methods. Our model exhibits competitive results in several real-world applications.
\end{itemize}


\vspace{-3.mm}
\section{Preliminary}
\subsection{Mutual Information Variational Bounds}
Mutual information (MI) is a key concept in information theory, for measuring the dependence between two random variables. Given two random variables $\vx$ and $\vy$, their MI is defined as 
\begin{equation}\label{eq:mi-definition}
    \MI(\vx; \vy) = \bbE_{p(\vx, \vy)} [\log \frac{p(\vx, \vy)}{p(\vx) p(\vy)} ],
\end{equation}
where $p(\vx, \vy)$ is the joint distribution of the random variables, with $p(\vx)$ and $p(\vy)$ representing the respective marginal distributions. 

In disentangled representation learning, a common goal is to minimize the MI between different types of embeddings~\citep{poole2019variational}. However, the exact  MI value is difficult to calculate in practice, because in most cases the integral in Eq.~\eqref{eq:mi-definition} is intractable. To address this problem, various MI estimation methods have been  introduced~\citep{chen2016infogan,belghazi2018mutual,poole2019variational}. One of the commonly used estimation approaches is the Barber-Agakov lower bound~\citep{barber2003algorithm}. By introducing a variational distribution $q(\vx| \vy)$, one may derive
\begin{equation}\label{eq:MI_variation_lowerbound} \textstyle
    \MI(\vx ; \vy) \geq \Enpy(\vx) + \bbE_{p(\vx,\vy)} [ \log q(\vx| \vy)],
\end{equation}
where $\Enpy(\vx)= \bbE_{p(\vx)}[-\log p(\vx)]$ is the entropy of variable $\vx$.
\subsection{Variation of Information}\label{sec:vi}
In information theory, Variation of Information (VI, also called Shared Information Distance) is a measure of independence between two random variables. The mathematical definition of VI between random variables $\vx$ and $\vy$ is
\begin{equation}\label{eq:vi_define}
\begin{aligned}\textstyle
    \VI(\vx; \vy)  = \Enpy(\vx) + \Enpy (\vy) - 2\MI(\vx; \vy), 
\end{aligned}
\end{equation}
where $\Enpy(\vx)$ and $ \Enpy(\vy)$ are entropies of $\vx$ and $\vy$, respectively (shown in Figure~\ref{fig:vi_def}).
\citet{kraskov2005hierarchical} show that VI is a well-defined metric, which satisfies the triangle inequality:
\begin{equation}\label{eq:triangle_ineq} \textstyle
    \VI(\vy; \vx) + \VI(\vx;  \vz) \geq \VI(\vy; \vz),
\end{equation}
for any random variables $\vx$, $\vy$ and $\vz$. Additionally, $\VI(\vx ;\vy) = 0$ indicates $\vx$ and $\vy$ are  the same variable~\citep{meilua2007comparing}.
From Eq.~\eqref{eq:vi_define}, the VI distance has a close relation to mutual information: if the mutual information is a measure of ``dependence'' between two variables, then the VI distance is a measure of ``independence'' between them.  

\begin{figure}[t]
    \centering
    \includegraphics[width=0.7\columnwidth]{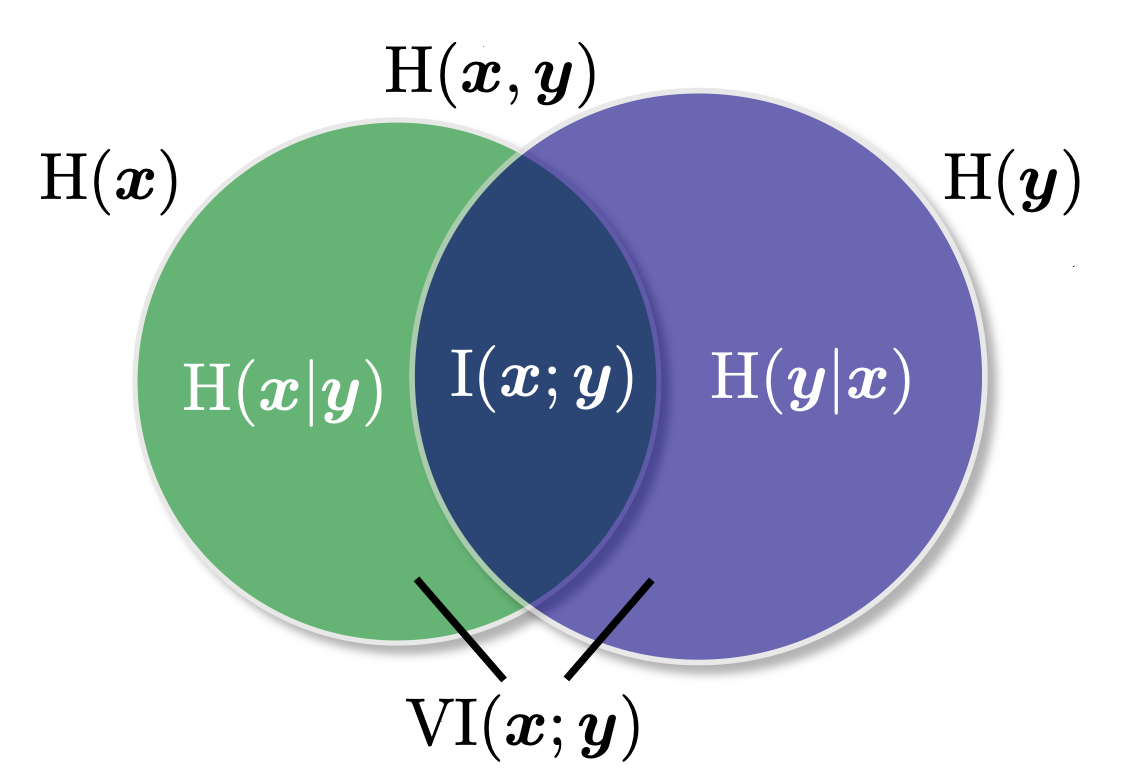}
    \caption{The green and purple circles represent the entropy of $\vx$ and $\vy$, respectively. The intersection (blue region) is the mutual information between $\vx$ and $\vy$. The symmetric difference of the two circles (green and purple regions) is $\VI (\vx; \vy)$.}
    \label{fig:vi_def}
\end{figure}

\section{Method}
Consider data $\{(\vx_i, y_i)\}_{i = 1}^N$, where each $\vx_i$ is a sentence drawn from a distribution $p(\vx)$, and $y_i$ is the label indicating the style of $\vx_i$.  The goal is to encode each sentence $\vx_i$ into its corresponding style embedding $\vs_i$ and content embedding $\vc_i$ with an encoder $q_\theta(\vs, \vc | \vx)$:
\begin{equation}\label{eq:encoder}\textstyle
  \vs_i,\vc_i| \vx_i \sim q_\theta(\vs, \vc| \vx_i).  
\end{equation}
The collection of style embeddings $\{\vs_i\}_{i=1}^N$ can be regarded as samples drawn from a variable $\vs$ in the style embedding space, while the collection of content embeddings $\{\vc_i \}_{i=1}^N$ are samples from a variable $\vc$ in the content embedding space. 
In practice, the dimension of the content embedding is typically higher than that of the style embedding, considering that the content usually contains more information than the style~\citep{john2018disentangled}.

We first give an intuitive introduction to our proposed VI-based objective, then in Section~\ref{sec:justification} we provide the theoretical justification for it.
To disentangle the style and content embedding, we try to minimize the mutual information $\MI (\vs; \vc)$ between $\vs$ and $\vc$. 
 Meanwhile, we maximize $\MI(\vc; \vx)$ to ensure that the content embedding $\vs$  sufficiently encapsulates information from the sentence $\vx$. 
  The embedding $\vs$ is expected to contain rich style information. Therefore, the mutual information $\MI(\vs; y)$ should be maximized. Thus, our overall disentangled representation learning objective is:
    $ \calL_{\text{Dis}} = \MI(\vs; \vc) - \MI ( \vc; \vx)  -  \MI(\vs; y).$



\subsection{Theoretical Justification of the Objective}\label{sec:justification} 
The objective $\calL_{\text{Dis}}$ has a strong connection with the independence measurement in information theory. As described in Section~\ref{sec:vi}, Variation of Information (VI) is a well-defined metric of independence between variables. Applying the triangle inequality from Eq.~\eqref{eq:triangle_ineq} to $\vs$, $\vc$ and $\vx$, we have $\VI(\vs; \vx) + \VI(\vx; \vc) \geq \VI(\vs; \vc).$
Equality occurs if and only if the information from variable $\vx$ is totally separated into two independent variable $\vs$ and $\vc$, which is an ideal scenario for disentangling sentence $\vx$ into its corresponding style embedding $\vs$ and content embedding $\vc$.

Therefore, the difference between $\VI(\vs; \vx) + \VI(\vx; \vc)$ and $\VI(\vs; \vc)$ represents the degree of disentanglement. Hence we introduce a measurement:
\begin{equation*} \textstyle
    {D}(\vx ; \vs, \vc) = \VI(\vs;\vx) + \VI(\vx; \vc) - \VI(\vc ; \vs).
\end{equation*}
From Eq.~\eqref{eq:triangle_ineq}, we know that $\text{D}(\vx ; \vy, \vz)$ is always non-negative.
By the definition of VI in Eq.~\eqref{eq:vi_define}, ${D}(\vx; \vs,\vc)$ can be simplified as:
\begin{equation*}\textstyle
\begin{aligned}
       & \VI(\vc ;\vx) + \VI(\vx; \vs) - \VI(\vs; \vc) \\
     = & 2 \Enpy(\vx) + 2[ \MI(\vs; \vc) - \MI(\vx; \vc) - \MI(\vx; \vs) ].
\end{aligned}
\end{equation*}
Since $\Enpy(\vx)$ is a constant associated with the data, we only need to focus on $ \MI(\vs; \vc) - \MI(\vx; \vc) - \MI(\vx; \vs)$. 

The measurement $D(\vx ; \vs, \vc)$ is symmetric to style $\vs$ and content $\vc$, giving rise to the problem that without any inductive bias in supervision, the disentangled representation could be meaningless (as observed by~\citet{locatello2019challenging}). Therefore, we add inductive biases by utilizing the style label $y$ as supervised information for style embedding $\vs$. 
Noting that $\vs \rightarrow \vx \rightarrow y$ is a Markov Chain, we have $\MI(\vs; \vx) \geq \MI(\vs; y)$ based on the MI data-processing inequality~\citep{cover2012elements}.
Then we convert the minimization of $ \MI(\vs; \vc) - \MI(\vx; \vc) - \MI(\vx; \vs)$ into the minimization of the upper bound $\MI(\vs; \vc) - \MI(\vx; \vc) - \MI(y; \vs)$, which further leads to our objective $\calL_{\text{Dis}}$. 

However, minimizing the exact value of mutual information in the objective $\calL_{\text{Dis}}$ causes numerical instabilities, especially when the dimension of the latent embeddings is large~\citep{chen2016infogan}. Therefore, we provide several MI estimations to the objective terms $\MI(\vs ; \vc)$, $\MI(\vx; \vc)$ and $\MI(\vs; y)$ in the following two sections.

\subsection{MI Variational Lower Bound}

To maximize $\MI(\vx; \vc)$ and $\MI(\vs; y)$, we derive two variational lower bounds. For $\MI(\vx; \vc)$, we introduce a variational decoder $q_\phi(\vx| \vc)$ to reconstruct the sentence $\vx$  by the content embedding $\vc$. Leveraging the MI variational lower bound from Eq.~\eqref{eq:MI_variation_lowerbound}, we have $\MI(\vx; \vc) \geq \Enpy(\vx)  + \bbE_{p(\vx; \vc)} [\log q_\phi(\vx | \vc)].$
%
Similarly, for $\MI(\vs; y)$, another variational lower bound  can be obtained as: $\MI(\vs; y) \geq \Enpy(y) + \bbE_{p(y, \vs)}[\log q_\psi(y|\vs)]$, where $q_\psi(y|\vs)$ is a classifier mapping the style embedding $\vs$ to its corresponding style label $y$.
Based on these two lower bounds, $\calL_{\text{Dis}}$ has an upper bound:
\begin{align}
\textstyle
    \calL_{\text{Dis}} \leq \MI(\vs; \vc)& -[\Enpy(\vx) + \bbE_{p(\vx, \vc)}[\log q_\phi(\vx|\vc)]] \nonumber \\
\textstyle    -&[\Enpy(y) + \bbE_{p(y, \vs)}[\log q_\psi(y|\vs)]]. \label{eq:lower-bound-in-obj}
\end{align}
Noting that both $\Enpy(\vx)$ and $\Enpy(y)$ are constants from the data, we only need to minimize: 
\begin{align}
\textstyle    \bar{\calL}_{\text{Dis}} = \MI(\vs; \vc) - &\bbE_{p(\vx, \vc)}[\log q_\phi(\vx|\vc)]\nonumber\\ 
\textstyle   -& \bbE_{p(y, \vs)}[\log q_\psi(y|\vs)]. \label{eq:bar_L_dis}
\end{align}
As an intuitive explanation of $\bar{\calL}_{\text{Dis}}$, the style embedding $\vs$ and content embedding $\vc$ are expected to be independent by minimizing mutual information $\MI(\vs; \vc)$, while they also need to be representative: the style embedding $\vs$ is encouraged to give a better prediction of style label $y$ by maximizing $\bbE_{p(y, \vs)}[\log q_\psi(y|\vs)]$; the content embedding should maximize the log-likelihood $ \bbE_{p(\vx, \vc)}[\log q_\phi(\vx|\vc)]$ to contain sufficient information from sentence $\vx$.

\subsection{MI Sample-based Upper Bound}\label{sec:upper_bound}
To estimate $\MI(\vs; \vc)$, we propose a novel sample-based upper bound.
Assume we have $M$ latent embedding pairs $\{(\vs_j, \vc_j)\}_{j=1}^M$ drawn from $p(\vs, \vc)$. As shown in Theorem~\ref{thm:upper-bound}, we derive an upper bound of mutual information based on the samples. A detailed proof is provided in the Supplementary Material.
\begin{theorem}\label{thm:upper-bound}
If $\{(\vs_j, \vc_j)\}_{j=1}^M \sim p(\vs,\vc)$, then
\begin{equation}\label{eq:mi-upper-bound}
\textstyle \mathrm{I}(\vs; \vc) \leq \bbE [ \frac{1}{M} \sum_{j = 1}^M R_j ] = : \hat{\mathrm{I}}(\vs; \vc), 
\end{equation}
where
$R_j = \log p(\vs_j | \vc_j) - \frac{1}{M} \sum_{k = 1}^M \log p(\vs_j|\vc_k) $.
\end{theorem}

Based on Theorem~\ref{thm:upper-bound}, given embedding samples $\{\vs_j, \vc_j \}_{j=1}^M$, we can minimize $\frac{1}{M} \sum_{j=1}^M R_j$ as an unbiased estimation of the upper bound $\hat{\MI}(\vs; \vc)$. The calculation of $R_{j}$ requires the conditional distribution $p(\vs| \vc)$, whose closed form is unknown. Therefore, we use a variational network $p_\sigma (\vs| \vc)$ to approximate $p(\vs| \vc)$ with embedding samples.


\begin{algorithm}[!t]
\small
\SetAlgoLined
\KwIn{Data $\{ \vx_j \}_{j=1}^M$, encoder $q_\theta(\vs, \vc | \vx)$, approximation network $p_\sigma(\vs| \vc)$.}
\For{each training iteration}{
 Sample $\{\vs_j, \vc_j\}_{j=1}^M$ from $q_\theta(\vs,\vc| \vx)$\;
  $\calL(\sigma) = \frac{1}{M} \sum_{j=1}^M \log p_\sigma(\vs_j | \vc_j)$\;
 Update $p_\sigma(\vs| \vc)$ by maximize $\calL(\sigma)$\;
 \For{$j = 1$ \KwTo $M$}{
  Sample $k'$ uniformly from $\{1,2,\dots,M \}$\;
  $\hat{R}_j = \log p_\sigma(\vs_j | \vc_j) - \log p_\sigma( \vs_j| \vc_{k'}) $\;}
 Update $q_\theta(\vs, \vc| \vx)$ by minimize $\frac{1}{M} \sum_{j=1}^M \hat{R}_j$\;
 }
 \caption{Disentangling $\vs$ and $\vc$}

 \end{algorithm}

To implement the upper bound in Eq.~\eqref{eq:mi-upper-bound}, we first feed $M$ sentences $\{ \vx_j\}$ into encoder $q_\theta(\vs, \vc| \vx)$ to  obtain embedding pairs $\{(\vs_j, \vc_j)\}$. Then, we train the variational distribution $p_\sigma(\vs | \vc)$ by maximizing the log-likelihood $\calL(\sigma) = \frac{1}{M} \sum_{j=1}^M \log p_\sigma (\vs_j | \vc_j)$. After the training of $p_\sigma(\vs| \vc)$ is finished, we calculate $R_j$ for each embedding pair $(\vs_j , \vc_j)$. Finally, the gradient for $\frac{1}{M} \sum_{j=1}^M R_j$ is calculated and back-propagated to encoder $q_\theta(\vs, \vc| \vx)$.  We apply the re-parameterization trick~\citep{kingma2013auto} to ensure the gradient back-propagates through the sampled embeddings $(\vs_j, \vc_j)$. When the encoder weights are updated, the distribution $q_\theta(\vs, \vc| \vx)$ changes, which leads to the changing of conditional distribution $p(\vs | \vc)$. Therefore, we need to update the approximation network $p_\sigma(\vs| \vc)$ again.
Consequently, the encoder network $q_\theta(\vs, \vc| \vx)$ and the approximation network $p_\sigma(\vs| \vc)$ are updated alternately during training.

In each training step, the above algorithm requires $M$ pairs of embedding samples $\{ \vs_j, \vc_j\}_{j=1}^M$ and the calculation of all  conditional distributions $p_\sigma(\vs_j| \vc_k)$. This leads to $\mathcal{O}(M^2)$ computational complexity. To accelerate the training, we further approximate term $\frac{1}{M} \sum_{k = 1}^M \log p(\vs_j|\vc_k)$ in $R_j$ by $\log p(\vs_j | \vc_{k'})$, where $k'$ is selected uniformly from indices $\{ 1, 2, \dots, M \}$. This stochastic sampling not only leads to an unbiased estimation $\hat{R}_j$ to $R_j$, but also improves the model robustness (as shown in Algorithm 1).

Symmetrically, we can also derive an MI upper bound based on the conditional distribution $p(\vc | \vs)$. However, the dimension of $\vc$ is much higher than the dimension of $\vs$, which indicates that
the neural approximation to $p(\vc | \vs)$ would have worse performance compared with the approximation to $p(\vs| \vc)$. Alternatively, the lower-dimensional distribution $p(\vs| \vc)$ used in our model is relatively easy to approximate with neural networks.

\subsection{Encoder-Decoder Framework}

One important downstream task for disentangled representation learning (DRL) is conditional generation. Our MI-based text DRL method can be also embedded into an Encoder-Decoder generative model and trained end-to-end.
 
Since the proposed DRL encoder $q_\theta (\vs, \vc | \vx)$ is a stochastic neural network, a natural extension is to add a decoder to build a variational autoencoder (VAE)~\citep{kingma2013auto}. Therefore, we introduce another decoder network $p_\gamma(\vx| \vs, \vc)$ that generates a new sentence based on the given style $\vs$ and content $\vc$.  A prior distribution $p(\vs, \vc)$ = $p(\vs)p(\vc)$, as the product of two multivariate unit-variance Gaussians, is used to regularize the posterior distribution $q_\theta ( \vs , \vc| \vx)$ by KL-divergence minimization. Meanwhile, the log-likelihood term for text reconstruction should be maximized. The objective for VAE is:
\begin{equation*}\textstyle
\begin{aligned}
    \calL_{\text{VAE}} = &\text{KL}( q_\theta(\vs, \vc| \vx) \Vert p(\vs,\vc)) \\ &- \bbE_{q_\theta(\vs, \vc | \vx )}[\log p_\gamma(\vx| \vs, \vc)].
\end{aligned}
\end{equation*}
We combine the VAE objective and our MI-based disentanglement term to form an end-to-end learning framework (as shown in Figure~\ref{fig:framework}). The total loss function is 
$\calL_{\text{total}} = \beta \calL_{\text{Dis}}^* + \calL_{\text{VAE}}$, 
\begin{figure}[t]
    \centering
    \includegraphics[width=\columnwidth]{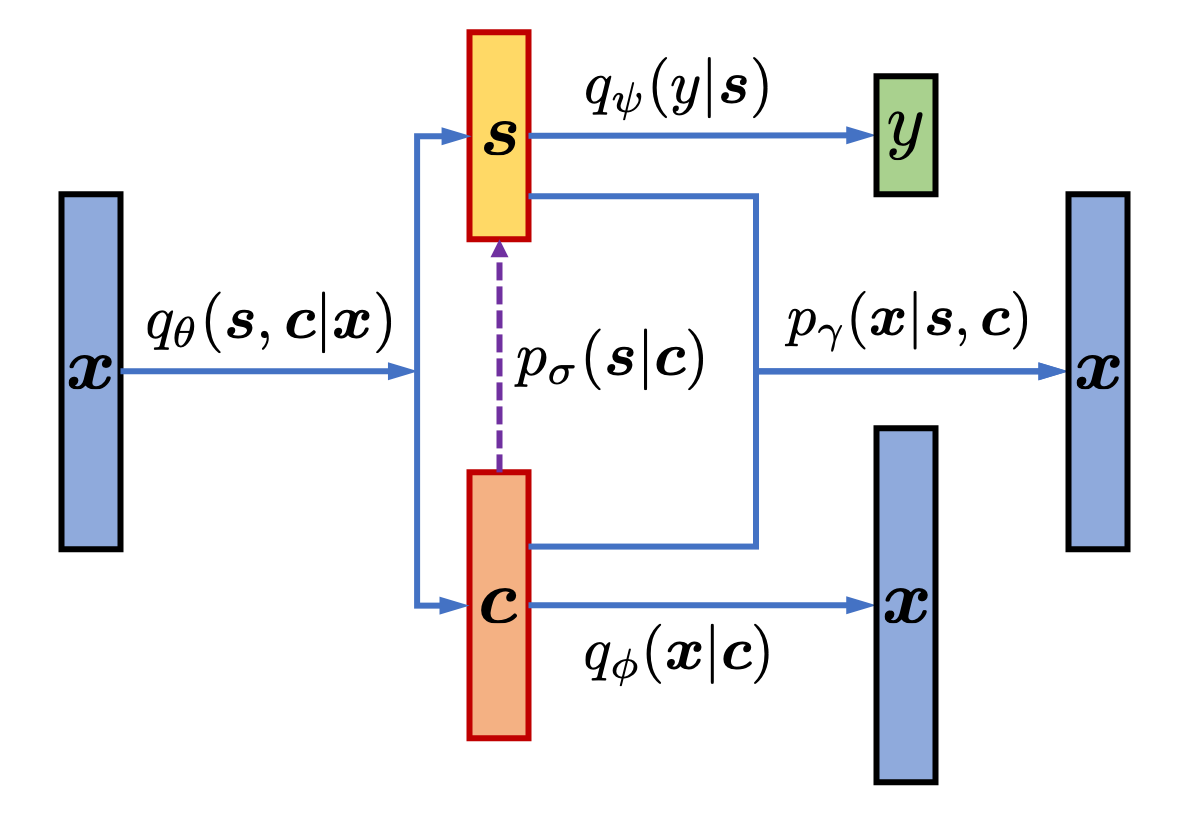}
    \caption{Proposed framework: Each sentence $\vx$ is encoded into style embedding $\vs$ and content embedding $\vc$. The style embedding $\vs$ goes through a classifier $q_\psi(y|\vs)$ to predict the style label $y$; the content embedding $\vc$ is used to reconstruct $\vx$. An auxiliary network $p_\sigma(\vs| \vc)$ helps disentangle the style and content embeddings. The decoder $p_\gamma(\vx| \vs, \vc)$ generates sentences based on the combination of $\vs$ and $\vc$. }
    \label{fig:framework}
\end{figure}
where $\calL^*_{\text{Dis}}$ replaces $\MI(\vs; \vc)$ in $\bar{\calL}_{\text{Dis}}$ (Eq.~\eqref{eq:bar_L_dis}) with our MI upper bound $\hat{\MI}(\vs; \vc)$ from Eq.~\eqref{eq:mi-upper-bound}; $\beta > 0$ is a hyper-parameter re-weighting DRL and VAE terms. 
We call this final framework Information-theoretic  Disentangled text Embedding Learning (IDEL). 

 \section{Related Work}
\subsection{Disentangled Representation Learning}
Disentangled representation learning (DRL) can be classified into two categories: unsupervised  disentangling and supervised disentangling. Unsupervised disentangling methods focus on adding constraints on the embedding space to enforce that each dimension of the space be as independent as possible~\citep{burgess2018understanding,chen2018isolating}. However, \citet{locatello2019challenging} challenge the effectiveness of unsupervised disentangling  without  any  induced bias from data or supervision. For supervised disentangling, supervision is always provided on different parts of disentangled representations. However, for text representation learning, supervised information can typically be provided only for the style embeddings (\textit{e.g.} sentiment or personality labels), making the task much more challenging. \citet{john2018disentangled} tried to alleviate this issue by manually removing sentiment-related words from a sentence. In contrast, our model is trained in an end-to-end manner without manually adding any supervision on the content embeddings.

\subsection{Mutual Information Estimation}
Mutual information (MI) is a fundamental measurement of the dependence between two random variables. MI has been applied to a wide range of tasks in machine learning, including generative modeling~\citep{chen2016infogan}, the information bottleneck~\citep{tishby2000information}, and domain adaptation~\citep{gholami2020unsupervised}. In our proposed method, we utilize MI to measure the dependence between content and style embedding. By minimizing the MI, the learned content and style representations are explicitly disentangled. 

However, the exact value of MI is hard to calculate, especially for high-dimensional embedding vectors~\citep{poole2019variational}. To approximate MI, most previous work focuses on lower-bound estimations~\citep{chen2016infogan,belghazi2018mutual,poole2019variational}, which are not applicable to MI minimization tasks.  \citet{poole2019variational} propose a leave-one-out upper bound of MI; however it is not numerically stable in practice. 
Inspired by these observations, we introduce a novel MI upper bound for disentangled representation learning, which stably minimizes the correlation between content and style embedding in a principled manner.

\section{Experiments}

\subsection{Datasets}
We conduct experiments to evaluate our models on the following real-world datasets:

     \textbf{Yelp Reviews:} The Yelp dataset contains online service reviews with associated rating scores.
    We follow the pre-processing from \citet{shen2017style} for a fair comparison. The resulting dataset includes 250,000 positive review sentences and 350,000 negative review sentences.
    
 \textbf{Personality Captioning:}
    Personality Captioning dataset~\citep{shuster2019engaging} collects captions of images which are written according to 215 different personality traits. These traits can be divided into three categories: \textit{positive}, \textit{neutral}, and \textit{negative}. We  select sentences from \textit{positive} and \textit{negative} classes for evaluation.
    
\subsection{Experimental Setup}
 We build the sentence encoder $q_\theta(\vs, \vc | \vx)$ with a one-layer bi-directional LSTM plus a multi-head attention mechanism. The style classifier $q_\psi(y|\vs)$ is parameterized by a single fully-connected network with the softmax activation. The content-based decoder $q_\phi(\vx | \vc)$ is a one-layer uni-directional LSTM appended with a linear layer with vocabulary size output, outputting the predicted probability of the next words. The conditional distribution approximation  $p_\sigma(\vs | \vc)$ is represented by a two-layer fully-connected network with ReLU activation. The generator $p_\gamma(\vx | \vs, \vc)$ is built by a two-layer uni-directional LSTM plus a linear projection with output dimension equal to the vocabulary size, providing the next-word prediction based on previous sentence information and the current word.

\begin{figure}[t]
    \centering
    \includegraphics[width = 0.95\columnwidth]{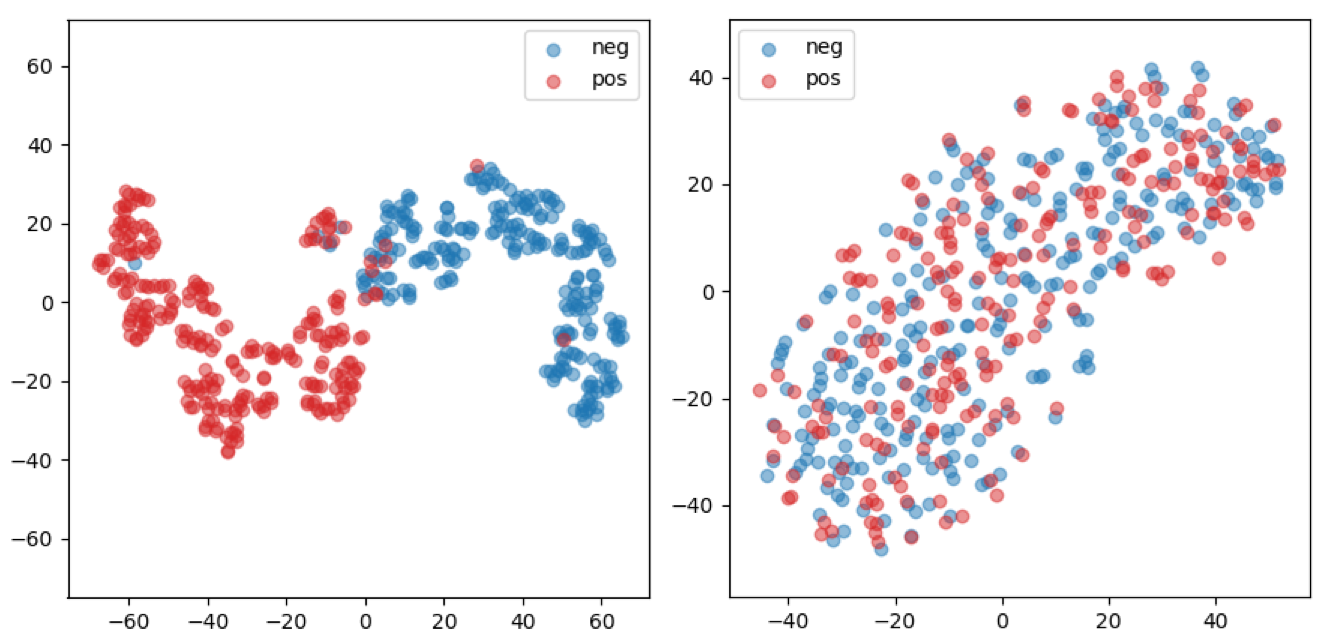}
    \caption{ Latent spaces t-SNE plots of IDEL on Yelp.}
    \label{fig:tsne_mi}
\end{figure}
\begin{figure}[t]
    \centering
    \includegraphics[width = 0.95\columnwidth]{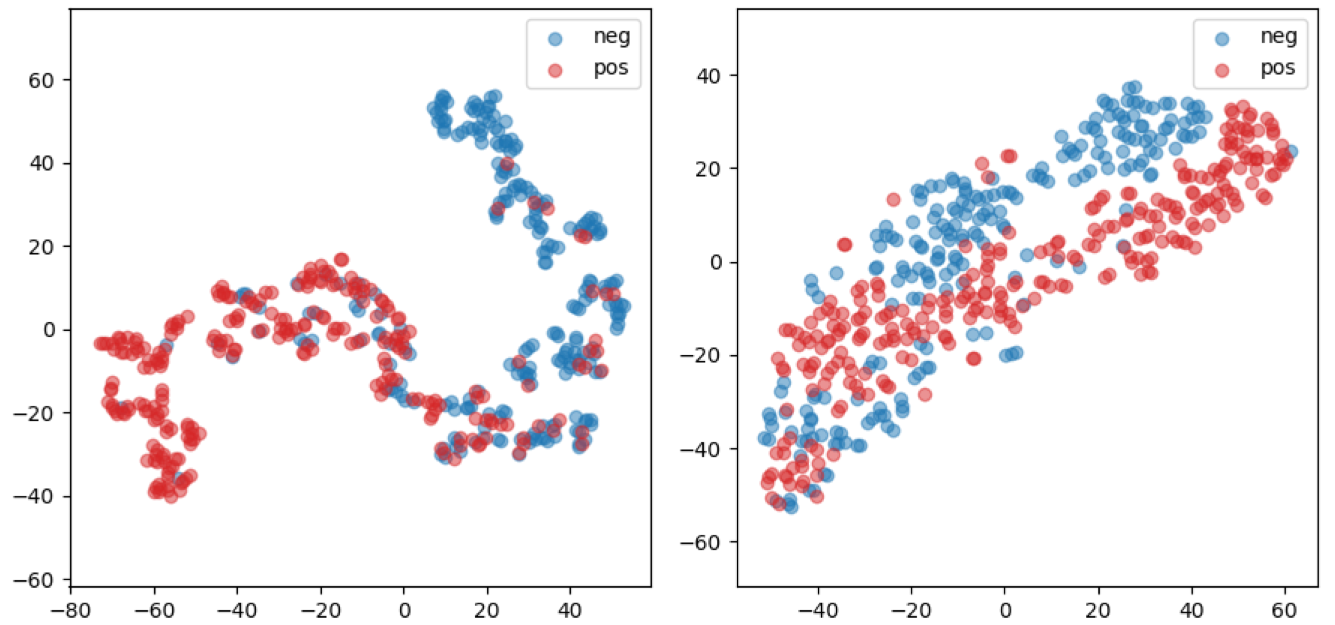}
    \caption{t-SNE plots of IDEL$^-$ without $\hat{\MI}(\vs; \vc)$.}
    \label{fig:tsne_no_mi}
\end{figure}

We initialize and  fix our word embeddings by the 300-dimensional pre-trained GloVe vectors~\citep{pennington2014glove}. The style embedding dimension is set to 32  and the content embedding dimension is 512.   We use a standard multivariate normal distribution as the prior of the latent spaces. We train the model with the Adam optimizer \citep{Adam} with initial learning rate of $5 \times 10^{-5}$. The batch size is equal to 128.

\subsection{Embedding Disentanglement Quality}\label{sec:DRL_latent_experiment}
We first examine the disentangling quality  of learned latent embeddings, primarily studying the latent spaces of IDEL on the Yelp dataset.

\textbf{Latent Space Visualization:} We randomly select 1,000 sentences from the Yelp testing set and visualize their latent embeddings  in Figure~\ref{fig:tsne_mi}, via t-SNE plots \citep{tSNE}. The blue and red points respectively represent the positive and negative sentences.  The left side of the figure shows the style embedding space, which is well separated into two parts with different colors. It supports the claim that our model learns a semantically meaningful style embedding space.  The right side of the figure is the content embedding space, which cannot be distinguished by the style labels (different colors). The lack of difference in the pattern of content embedding also provides evidence that our content embeddings have little correlation with the style labels.

For an ablation study, we train another IDEL model under  the same setup, while removing our MI upper bound $\hat{\MI}(\vs; \vc)$. We call this model IDEL$^-$ in the following experiments.  We encode the same sentences used in Figure~\ref{fig:tsne_mi}, and display the corresponding embeddings in Figure~\ref{fig:tsne_no_mi}. Compared with results from the original IDEL, the style embedding space (left in Figure~\ref{fig:tsne_no_mi}) is not separated in a clean manner. On the other hand, the positive and negative embeddings become distinguishable in the content embedding space. The difference between Figures~\ref{fig:tsne_mi} and \ref{fig:tsne_no_mi} indicates the disentangling effectiveness of our MI upper bound $\hat{\MI}(\vs; \vc)$.



\textbf{Label-Embedding Correlation:} 
Besides visualization, we also numerically analyze the correlation between latent embeddings and style labels. 
Inspired by the statistical two-sample test~\citep{gretton2012kernel}, we use the sample-based divergence between the positive embedding distribution $p(\vc | y = 1)$ and the negative embedding distribution $p(\vc |y = 0)$ as a measurement of label-embedding correlation.  
We consider four divergences: Mean Absolute Deviation (MAD)~\citep{geary1935ratio}, Energy Distance (ED)~\citep{sejdinovic2013equivalence}, Maximum Mean Discrepancy (MMD)~\citep{gretton2012kernel}, and Wasserstein distance (WD)~\citep{ramdas2017wasserstein}. For a fair comparison, we re-implement previous text embedding methods and set their content embedding dimension to 512 and the style embedding dimension to 32 (if applicable). Details about the divergences and embedding processing are shown in the Supplementary Material. 

\begin{table*}[htbp]
  \centering
  \resizebox{\textwidth}{!}{
    \begin{tabular}{l|ccc|cccc|ccc|cccc}
        \toprule[1.2pt]
        & \multicolumn{7}{c|}{\textbf{Yelp Dataset}} & \multicolumn{7}{c}{\textbf{Personality Captioning Dataset}}\\
         \hline
          & \multicolumn{3}{c|}{Conditional Generation} & \multicolumn{4}{c|}{Style Transfer} &  \multicolumn{3}{c|}{Conditional Generation} & \multicolumn{4}{c}{Style Transfer} \\
    \hline
          & \textbf{ACC} & \textbf{BLEU} & \textbf{GM} &  \textbf{ACC} & \textbf{BLEU} & \textbf{S-BLEU} & \textbf{GM} & \textbf{ACC} & \textbf{BLEU} & \textbf{GM} &  \textbf{ACC} & \textbf{BLEU} & \textbf{S-BLEU} & \textbf{GM} \\
    \hline
    \textbf{CtrlGen} & 82.5 & 20.8 & 41.4  & 83.4 & 19.4 & 31.4 & 37.0 & 73.6 & 18.9 & 37.0  & 73.3 & 18.9 & 30.0 & 34.6 \\
    \textbf{CAAE}    & 78.9 & 19.7 & 39.4  & 79.3 & 18.5 & 28.2 &34.6& 72.2 & 19.5 & 37.5  & 72.1 & 18.3 & 27.4 & 33.1 \\
    \textbf{ARAE}    & 78.3 & \textbf{23.1} & 42.4  & 78.5 & 21.3 & 32.5 & 37.9   & 72.8 & \textbf{22.5} & 40.4 & 71.5 & 20.4 & 31.6 & 35.8 \\
    \textbf{BT}      & 81.4 & 20.2 & 40.5  & \textbf{86.3} & 24.1 & \textbf{35.6} & \textbf{41.9}  &  74.1 & 21.0 & 39.4  & \textbf{75.9} & 23.1 & 34.2 & 39.1  \\
    \textbf{DRLST}   & 83.7 & 22.8 &  43.7  & 85.0 & 23.9 & 34.9 & 41.4 & 74.9 & {22.0} & {40.5}  & 75.7 & 21.9 & 33.8 & 38.3\\
    \hline
    \textbf{IDEL$^-$} & 78.1 & 20.3& 39.8  & 79.1& 20.1 & 27.5 & 35.1 & 72.0 & 19.7& 37.7  & 72.4 & 19.7 & 27.1 & 33.8\\
    \textbf{IDEL}     & \textbf{83.9} & 23.0 & \textbf{43.9}  & 85.7 & \textbf{24.3} & 35.2 & \textbf{41.9}  & \textbf{75.1} & 22.3& \textbf{40.9}  & 75.6 & \textbf{23.3} & \textbf{34.6} & \textbf{39.4}\\
  \bottomrule[1.2pt]
    \end{tabular}}
      \caption{Performance comparison of text DRL models. For conditional generation, the GM scores are calculated over ACC  and BLEU. For style transfer, the GMs are calculated over ACC, BLEU, S-BLEU(self-BLEU).}
  \label{tab:generation_results}%
\end{table*}%

From Table~\ref{tab:content_emb_dist}, the proposed IDEL achieves the lowest divergences between positive and negative \textit{content} embeddings compared with CtrlGen~\citep{hu2017toward}, CAAE~\citep{shen2017style}, ARAE~\citep{zhao2017adversarially}, BackTranslation (BT)~\citep{lample2018multipleattribute}, and DRLST~\citep{john2018disentangled}, indicating our model better disentangles the content embeddings from the style labels. For \textit{style} embeddings, we compare IDEL with DRLST, the only prior method that infers the text style embeddings. Table~\ref{tab:style_emb_dist} shows a larger distribution gap between  positive and negative style embeddings with IDEL than with DRLST, which demonstrates the proposed IDEL has better style information expression in the style embedding space. The comparison between IDEL and IDEL$^-$ supports the effectiveness of our MI upper bound minimization.

\begin{table}[t]  \footnotesize 
	\centering
	\scalebox{0.85}{
	\begin{tabular}{lccccc}
 		\toprule[1.2pt]
 
		\textbf{Method} & \textbf{MAD} &  	\textbf{ED} & \textbf{WD}  & 	\textbf{MMD}\\
		\midrule
		 \textbf{CtrlGen}      & 0.261 &0.105 & 0.311 &0.063   \\
        \textbf{CAAE}           & 0.285 &0.112 & 0.306 & 0.078         \\
        \textbf{ARAE}           & 0.194 & 0.050 & 0.248 & 0.042 \\
        \textbf{BT}& 0.211 & 0.053& 0.269 &0.049\\
        \textbf{DRLST}          &  0.181& 0.048& 0.215 &0.031 \\
    \midrule
      \textbf{IDEL$^-$}         &  0.217 & 0.077 & 0.293 & 0.051 \\
        \textbf{IDEL} & \textbf{0.063} & \textbf{0.015}& \textbf{0.084} &\textbf{0.010} \\
        \bottomrule[1.2pt]
	\end{tabular}}
	\caption{Sample divergences between positive and negative \textit{content} embeddings. }
	\label{tab:content_emb_dist}
\end{table}
\begin{table}[t]  \footnotesize 
	\centering
	\scalebox{0.85}{
	\begin{tabular}{lccccc}
 		\toprule[1.2pt]
 
		\textbf{Method} & \textbf{MAD} &  	\textbf{ED} & \textbf{WD}  & 	\textbf{MMD}\\
\midrule
        \textbf{DRLST}          &  1.024& 0.503& 1.375 &0.286 \\
        \textbf{IDEL$^-$}   &  0.996  & 0.489  & 1.124 & 0.251 \\
        \textbf{IDEL} & \textbf{1.167} & \textbf{0.583}& \textbf{1.392} &\textbf{0.302} \\
        \bottomrule[1.2pt]
	\end{tabular}}
	\caption{Sample divergences between positive and negative \textit{style} embeddings. }
	\label{tab:style_emb_dist}
\end{table}

\subsection{Embedding Representation Quality}
To show the representation ability of  IDEL, we conduct experiments on two text-generation tasks: style transfer and conditional generation. 

For style transfer, we encode two sentences into a disentangled representation, and then combine the style embedding from one sentence and the content embedding from another to generate a new sentence via the generator $p_\gamma(\vx | \vs, \vc)$. For conditional generation, we set one of the style or content embeddings to be fixed and sample the other part from the latent prior distribution, and then use the combination to generate text. Since most previous work only embedded the content information, for fair comparison,  we mainly focus on fixing style and sampling context embeddings under the conditional generation setup.

To measure generation quality for both tasks, we test the following metrics (more specific description is provided in the Supplementary Material).

\textbf{Style Preservation:} Following previous work~\citep{hu2017toward,shen2017style,john2018disentangled}, we pre-train a style classifier and use it to test whether a generated sentence can be categorized into the correct target style class.

\textbf{Content Preservation:} For style transfer, we measure whether a generation preserves the content information from the original sentence by the self-BLEU score~\citep{zhang2019text,ruiyi2020improving}. The self-BLEU is calculated between one original sentence and its style-transferred sentence.

\textbf{Generation Quality:} To measure the generation quality, we calculate the corpus-level BLEU score~\citep{papineni2002bleu} between a generated sentence and the testing data corpus.

\textbf{Geometric Mean:} We use the geometric mean (GM)~\citep{john2018disentangled} of the above metrics to obtain an overall evaluation metric of representiveness of DRL models.

\begin{table*}[htbp]
  \centering
  \resizebox{\textwidth}{!}{
    \begin{tabular}{l|l|l}
        \toprule[1.8pt]
           \textbf{Content Source} & \textbf{Style Source} & \textbf{Transferred Result}  \\
  \midrule[1.2pt]

          I \textbf{enjoy} it thoroughly! & never before had a \textbf{bad} experience at the habit until tonight. & I \textbf{dislike} it thoroughly.
      \\
      \hline
      
      quality is \textbf{just so so}. &  & quality is so \textbf{bad}. \\
      \hline
       I am so \textbf{grateful}.  & &  I am so \textbf{disgusted}. \\
   
    \midrule[1.2pt]
      
     never before had a \textbf{bad} experience at the habit until tonight. &  I am so \textbf{grateful}. & never had a service that was \textbf{enjoyable} experience tonight. \\
      \hline
      
     & quality is \textbf{just so so}.  & never had a \textbf{unimpressed} experience until tonight. \\
      \hline
      &quality of food is \textbf{fantastic}.   & never had \textbf{awesome} routine until tonight. \\
      \midrule[1.2pt]
     I am so \textbf{disappointed} with palm today.  & we were both so \textbf{impressed}.  & I am so \textbf{impressed} with palm again. \\
     \hline
      & quality of food is \textbf{fantastic} . &  I am \textbf{good} with palm today. \\
      \hline
      & never before had a \textbf{bad} experience at the habit until tonight. & I am so \textbf{disgusted} with palm today. \\
  \bottomrule[1.8pt]
    \end{tabular}}
      \caption{Examples of text style transfer on Yelp dataset.  The style-related words are bold.}
  \label{tab:sentence_examples}%
\end{table*}%

We compare our IDEL with previous state-of-the-art methods on Yelp and Personality Captioning datasets, as shown in Table~\ref{tab:generation_results}. The references to the other models are mentioned in Section~\ref{sec:DRL_latent_experiment}.  Note that the original BackTranslation (BT) method~\citep{lample2018multipleattribute} is a Auto-Encoder framework, that is not able to do conditional generation. To compare with BT fairly, we add a standard Gaussian prior in its latent space to make it a variational auto-encoder model.

From the results in Table~\ref{tab:generation_results}, ARAE performs well  on  the conditional generation. Compared to ARAE, our model performance is slightly lower on content preservation (BLEU). In contrast, the style classification score of IDEL has a large margin above that of ARAE.  The BackTranslation (BT) has a better performance on style transfer tasks, especially on the Yelp dataset. Our IDEL has a lower style classification accuracy (ACC) than BT on the style transfer task. However, IDEL achieves high BLEU on style transfer, which leads to a high overall GM score on the Personality-Captioning dataset. On the Yelp dataset, IDEL also has a competitive GM score compared with BT. The experiments show a clear trade-off between style preservation and content preservation, in which our IDEL learns more representative disentangled representation and leads to a better balance.

Besides the automatic evaluation metrics mentioned above, we further test our disentangled representation effectiveness by human evaluation. Due to the limitation of manual effort, we only evaluate the style transfer performance on Yelp datasets. The generated sentences are manually evaluated on style accuracy (SA), content preservation (CP), and sentence fluency (SF). The CP and SF scores are between 0 to 5. Details are provided in the Supplementary Material. Our method achieves better style and content preservation, with a little performance sacrifice on sentence fluency. 
\begin{table}[t]
  \centering
 
  \scalebox{0.76}{
    \begin{tabular}{l|cccc}
        \toprule[1.2pt]
          & \textbf{SA} & \textbf{CP} & \textbf{SF} & \textbf{GM}  \\
    \hline
   \textbf{CtrlGen} & 71.2 (3.56) & 3.25 & 3.12 & 3.30 \\
   \textbf{CAAE}    & 63.1 (3.16) & 2.83 & 3.06 & 3.01 \\
   \textbf{ARAE}    & 68.0 (3.40) & \textbf{3.44} & 3.09 & 3.31  \\
    \textbf{IDEL}   & \textbf{73.7 (3.69)} & 3.39 & \textbf{3.21} & \textbf{3.42} \\
  \bottomrule[1.2pt]
    \end{tabular}}
    \vspace{-1mm}
      \caption{Manual evaluation for style transfer on Yelp. The style accuracy (SA) scores are scaled in range $[0,5]$ for compatible calculation of geometric mean (GM).}
  \label{tab:addlabel}%
\end{table}%
\begin{table}[t]
\centering
  \scalebox{0.75}{
    \begin{tabular}{l|cccc}
        \toprule[1.2pt]
          & \textbf{ACC} & \textbf{BLEU} & \textbf{S-BLEU} & \textbf{GM}  \\
    \hline
    $\calL_{\text{VAE}}$  & 52.1 & \textbf{24.7} & 20.8 & 29.9 \\
    $\calL_{\text{VAE}} + \MI(\vs;y)$    & \textbf{86.1} & 23.3 & 16.4 & 32.0  \\
    $\calL_{\text{VAE}} + \MI (\vx; \vc)$    & 50.2 & 24.0 & \textbf{36.3} & 34.7 \\
    \textbf{IDEL$^-$} & 79.1 & 20.1 & 27.5 & 35.1 \\
    \textbf{IDEL$^*$} & 85.5 & 24.0 & 35.0 & 41.5  \\
    \textbf{IDEL} & 85.7 & 24.3 & 35.2 & \textbf{41.9}\\
  \bottomrule[1.2pt]
    \end{tabular}}
      \caption{Ablation tests for style transfer on Yelp.}
  \label{tab:Ablation}%
\end{table}%

 Table~\ref{tab:sentence_examples} shows three style transfer examples from IDEL on the Yelp dataset. The first example shows three sentences transferred with the style from a given sentence. The other two examples transfer each given sentence based on the styles of three different sentences. Our IDEL  not only transfers sentences into target sentiment classes, but also renders the sentence with more detailed style information (\textit{e.g.}, the degree of the sentiment).

In addition, we conduct an ablation study to test the influence of different objective terms in our model. We re-train the model with different training loss combinations while keeping all other setups the same. In Table~\ref{tab:generation_results}, IDEL surpasses IDEL$^-$ (without MI upper bound minimization) with a large gap, 
demonstrating the effectiveness of our proposed MI upper bound. The vanilla VAE has the best generation quality. However, its transfer style accuracy is slightly better than a random guess. When adding $\MI(\vs; y)$, the ACC score significantly improves, but the content preservation (S-BLEU) becomes worse. When adding $\MI(\vc; \vx)$, the content information is well preserved, while the ACC even decreases. By gradually adding MI terms, the model performance becomes more balanced on all the metrics, with the overall GM monotonically increasing. Additionally, we test the influence of the stochastic calculation of $R_j$ in Algorithm~1 (IDEL) with the closed form from Theorem~\ref{thm:upper-bound} (IDEL$^*$). The stochastic  IDEL not only accelerates the training but also  gains a performance improvement relative to IDEL$^*$.

\section{Conclusions}
We have proposed a novel information-theoretic disentangled text representation learning framework. Following the theoretical guidance from  information theory, our method separates the textual information into independent spaces, constituting style and content representations. A sample-based mutual information upper bound is derived to help reduce the dependence between embedding spaces. Concurrently, the original text information is well preserved by maximizing the mutual information between input sentences and latent representations. In experiments, we introduce several two-sample test statistics to measure label-embedding correlation. The proposed model achieves competitive performance compared with previous methods on both conditional generation and style transfer. For future work, our model can be extended to disentangled representation learning with non-categorical style labels, and applied to zero-shot style transfer with newly-coming unseen styles.

\subsection*{Acknowledgements} 
 This work was supported by NEC Labs America, and was conducted while the first author was doing an internship at NEC Labs America.
\bibliography{acl2020}

\begin{thebibliography}{38}
\expandafter\ifx\csname natexlab\endcsname\relax\def\natexlab#1{#1}\fi

\bibitem[{Barber and Agakov(2003)}]{barber2003algorithm}
David Barber and Felix~V Agakov. 2003.
\newblock The im algorithm: a variational approach to information maximization.
\newblock In \emph{Advances in neural information processing systems}, page
  None.

\bibitem[{Belghazi et~al.(2018)Belghazi, Baratin, Rajeshwar, Ozair, Bengio,
  Hjelm, and Courville}]{belghazi2018mutual}
Mohamed~Ishmael Belghazi, Aristide Baratin, Sai Rajeshwar, Sherjil Ozair,
  Yoshua Bengio, Devon Hjelm, and Aaron Courville. 2018.
\newblock Mutual information neural estimation.
\newblock In \emph{International Conference on Machine Learning}, pages
  530--539.

\bibitem[{Burgess et~al.(2018)Burgess, Higgins, Pal, Matthey, Watters,
  Desjardins, and Lerchner}]{burgess2018understanding}
Christopher~P Burgess, Irina Higgins, Arka Pal, Loic Matthey, Nick Watters,
  Guillaume Desjardins, and Alexander Lerchner. 2018.
\newblock Understanding disentangling in beta-vae.
\newblock \emph{arXiv preprint arXiv:1804.03599}.

\bibitem[{Chen et~al.(2018)Chen, Li, Grosse, and Duvenaud}]{chen2018isolating}
Tian~Qi Chen, Xuechen Li, Roger~B Grosse, and David~K Duvenaud. 2018.
\newblock Isolating sources of disentanglement in variational autoencoders.
\newblock In \emph{Advances in Neural Information Processing Systems}, pages
  2610--2620.

\bibitem[{Chen et~al.(2016)Chen, Duan, Houthooft, Schulman, Sutskever, and
  Abbeel}]{chen2016infogan}
Xi~Chen, Yan Duan, Rein Houthooft, John Schulman, Ilya Sutskever, and Pieter
  Abbeel. 2016.
\newblock Infogan: Interpretable representation learning by information
  maximizing generative adversarial nets.
\newblock In \emph{Advances in neural information processing systems}, pages
  2172--2180.

\bibitem[{Chou et~al.(2018)Chou, chieh Yeh, yi~Lee, and shan
  Lee}]{chou2018multi}
Juchieh Chou, Cheng chieh Yeh, Hung yi~Lee, and Lin shan Lee. 2018.
\newblock Multi-target voice conversion without parallel data by adversarially
  learning disentangled audio representations.
\newblock In \emph{Proc. Interspeech 2018}, pages 501--505.

\bibitem[{Cover and Thomas(2012)}]{cover2012elements}
Thomas~M Cover and Joy~A Thomas. 2012.
\newblock \emph{Elements of information theory}.
\newblock John Wiley \& Sons.

\bibitem[{Denton et~al.(2017)}]{denton2017unsupervised}
Emily~L Denton et~al. 2017.
\newblock Unsupervised learning of disentangled representations from video.
\newblock In \emph{Advances in neural information processing systems}, pages
  4414--4423.

\bibitem[{Geary(1935)}]{geary1935ratio}
Roy~C Geary. 1935.
\newblock The ratio of the mean deviation to the standard deviation as a test
  of normality.
\newblock \emph{Biometrika}, 27(3/4):310--332.

\bibitem[{Gholami et~al.(2020)Gholami, Sahu, Rudovic, Bousmalis, and
  Pavlovic}]{gholami2020unsupervised}
Behnam Gholami, Pritish Sahu, Ognjen Rudovic, Konstantinos Bousmalis, and
  Vladimir Pavlovic. 2020.
\newblock Unsupervised multi-target domain adaptation: An information theoretic
  approach.
\newblock \emph{IEEE Transactions on Image Processing}, 29:3993--4002.

\bibitem[{Gretton et~al.(2012)Gretton, Borgwardt, Rasch, Sch{\"o}lkopf, and
  Smola}]{gretton2012kernel}
Arthur Gretton, Karsten~M Borgwardt, Malte~J Rasch, Bernhard Sch{\"o}lkopf, and
  Alexander Smola. 2012.
\newblock A kernel two-sample test.
\newblock \emph{Journal of Machine Learning Research}, 13(Mar):723--773.

\bibitem[{Hsieh et~al.(2018)Hsieh, Liu, Huang, Fei-Fei, and
  Niebles}]{hsieh2018learning}
Jun-Ting Hsieh, Bingbin Liu, De-An Huang, Li~F Fei-Fei, and Juan~Carlos
  Niebles. 2018.
\newblock Learning to decompose and disentangle representations for video
  prediction.
\newblock In \emph{Advances in Neural Information Processing Systems}, pages
  517--526.

\bibitem[{Hu et~al.(2017)Hu, Yang, Liang, Salakhutdinov, and
  Xing}]{hu2017toward}
Zhiting Hu, Zichao Yang, Xiaodan Liang, Ruslan Salakhutdinov, and Eric~P Xing.
  2017.
\newblock Toward controlled generation of text.
\newblock In \emph{Proceedings of the 34th International Conference on Machine
  Learning-Volume 70}, pages 1587--1596. JMLR. org.

\bibitem[{John et~al.(2019)John, Mou, Bahuleyan, and
  Vechtomova}]{john2018disentangled}
Vineet John, Lili Mou, Hareesh Bahuleyan, and Olga Vechtomova. 2019.
\newblock Disentangled representation learning for non-parallel text style
  transfer.
\newblock \emph{Proceedings of the 57th Annual Meeting of the Association for
  Computational Linguistics}.

\bibitem[{Kingma and Ba(2014)}]{Adam}
Diederik~P. Kingma and Jimmy Ba. 2014.
\newblock Adam: A method for stochastic optimization.
\newblock \emph{arXiv:1412.6980v9}.

\bibitem[{Kingma and Welling(2013)}]{kingma2013auto}
Diederik~P Kingma and Max Welling. 2013.
\newblock Auto-encoding variational {B}ayes.
\newblock \emph{arXiv preprint arXiv:1312.6114}.

\bibitem[{Kraskov et~al.(2005)Kraskov, St{\"o}gbauer, Andrzejak, and
  Grassberger}]{kraskov2005hierarchical}
Alexander Kraskov, Harald St{\"o}gbauer, Ralph~G Andrzejak, and Peter
  Grassberger. 2005.
\newblock Hierarchical clustering using mutual information.
\newblock \emph{EPL (Europhysics Letters)}, 70(2):278.

\bibitem[{Kumar~Verma et~al.(2018)Kumar~Verma, Arora, Mishra, and
  Rai}]{kumar2018generalized}
Vinay Kumar~Verma, Gundeep Arora, Ashish Mishra, and Piyush Rai. 2018.
\newblock Generalized zero-shot learning via synthesized examples.
\newblock In \emph{Proceedings of the IEEE conference on computer vision and
  pattern recognition}, pages 4281--4289.

\bibitem[{Lample et~al.(2019)Lample, Subramanian, Smith, Denoyer, Ranzato, and
  Boureau}]{lample2018multipleattribute}
Guillaume Lample, Sandeep Subramanian, Eric Smith, Ludovic Denoyer,
  Marc'Aurelio Ranzato, and Y-Lan Boureau. 2019.
\newblock Multiple-attribute text rewriting.
\newblock In \emph{International Conference on Learning Representations}.

\bibitem[{Lee et~al.(2018)Lee, Tseng, Huang, Singh, and Yang}]{lee2018diverse}
Hsin-Ying Lee, Hung-Yu Tseng, Jia-Bin Huang, Maneesh Singh, and Ming-Hsuan
  Yang. 2018.
\newblock Diverse image-to-image translation via disentangled representations.
\newblock In \emph{Proceedings of the European Conference on Computer Vision
  (ECCV)}, pages 35--51.

\bibitem[{Liu et~al.(2018)Liu, Yeh, Fu, Wang, Chiu, and
  Frank~Wang}]{liu2018detach}
Yen-Cheng Liu, Yu-Ying Yeh, Tzu-Chien Fu, Sheng-De Wang, Wei-Chen Chiu, and
  Yu-Chiang Frank~Wang. 2018.
\newblock Detach and adapt: Learning cross-domain disentangled deep
  representation.
\newblock In \emph{Proceedings of the IEEE Conference on Computer Vision and
  Pattern Recognition}, pages 8867--8876.

\bibitem[{Locatello et~al.(2019)Locatello, Bauer, Lucic, Raetsch, Gelly,
  Sch{\"o}lkopf, and Bachem}]{locatello2019challenging}
Francesco Locatello, Stefan Bauer, Mario Lucic, Gunnar Raetsch, Sylvain Gelly,
  Bernhard Sch{\"o}lkopf, and Olivier Bachem. 2019.
\newblock Challenging common assumptions in the unsupervised learning of
  disentangled representations.
\newblock In \emph{International Conference on Machine Learning}, pages
  4114--4124.

\bibitem[{van~der Maaten and Hinton(2008)}]{tSNE}
Laurens van~der Maaten and Geoffrey Hinton. 2008.
\newblock Visualizing high-dimensional data using t-{SNE}.
\newblock \emph{JMLR}.

\bibitem[{Meil{\u{a}}(2007)}]{meilua2007comparing}
Marina Meil{\u{a}}. 2007.
\newblock Comparing clusterings—an information based distance.
\newblock \emph{Journal of multivariate analysis}, 98(5):873--895.

\bibitem[{Papineni et~al.(2002)Papineni, Roukos, Ward, and
  Zhu}]{papineni2002bleu}
Kishore Papineni, Salim Roukos, Todd Ward, and Wei-Jing Zhu. 2002.
\newblock {BLEU}: a method for automatic evaluation of machine translation.
\newblock In \emph{Proceedings of the 40th annual meeting on association for
  computational linguistics}, pages 311--318. Association for Computational
  Linguistics.

\bibitem[{Pennington et~al.(2014)Pennington, Socher, and
  Manning}]{pennington2014glove}
Jeffrey Pennington, Richard Socher, and Christopher Manning. 2014.
\newblock Glove: Global vectors for word representation.
\newblock In \emph{Proceedings of the 2014 conference on empirical methods in
  natural language processing (EMNLP)}, pages 1532--1543.

\bibitem[{Poole et~al.(2019)Poole, Ozair, Van Den~Oord, Alemi, and
  Tucker}]{poole2019variational}
Ben Poole, Sherjil Ozair, Aaron Van Den~Oord, Alex Alemi, and George Tucker.
  2019.
\newblock On variational bounds of mutual information.
\newblock In \emph{International Conference on Machine Learning}, pages
  5171--5180.

\bibitem[{Ramdas et~al.(2017)Ramdas, Trillos, and
  Cuturi}]{ramdas2017wasserstein}
Aaditya Ramdas, Nicol{\'a}s Trillos, and Marco Cuturi. 2017.
\newblock On {W}asserstein two-sample testing and related families of
  nonparametric tests.
\newblock \emph{Entropy}, 19(2):47.

\bibitem[{Sejdinovic et~al.(2013)Sejdinovic, Sriperumbudur, Gretton, Fukumizu
  et~al.}]{sejdinovic2013equivalence}
Dino Sejdinovic, Bharath Sriperumbudur, Arthur Gretton, Kenji Fukumizu, et~al.
  2013.
\newblock Equivalence of distance-based and rkhs-based statistics in hypothesis
  testing.
\newblock \emph{The Annals of Statistics}, 41(5):2263--2291.

\bibitem[{Shen et~al.(2017)Shen, Lei, Barzilay, and Jaakkola}]{shen2017style}
Tianxiao Shen, Tao Lei, Regina Barzilay, and Tommi Jaakkola. 2017.
\newblock Style transfer from non-parallel text by cross-alignment.
\newblock In \emph{Advances in neural information processing systems}, pages
  6830--6841.

\bibitem[{Shuster et~al.(2019)Shuster, Humeau, Hu, Bordes, and
  Weston}]{shuster2019engaging}
Kurt Shuster, Samuel Humeau, Hexiang Hu, Antoine Bordes, and Jason Weston.
  2019.
\newblock Engaging image captioning via personality.
\newblock In \emph{Proceedings of the IEEE Conference on Computer Vision and
  Pattern Recognition}, pages 12516--12526.

\bibitem[{Tishby et~al.(2000)Tishby, Pereira, and
  Bialek}]{tishby2000information}
Naftali Tishby, Fernando~C Pereira, and William Bialek. 2000.
\newblock The information bottleneck method.
\newblock \emph{arXiv preprint physics/0004057}.

\bibitem[{Tran et~al.(2017)Tran, Yin, and Liu}]{tran2017disentangled}
Luan Tran, Xi~Yin, and Xiaoming Liu. 2017.
\newblock Disentangled representation learning gan for pose-invariant face
  recognition.
\newblock In \emph{Proceedings of the IEEE Conference on Computer Vision and
  Pattern Recognition}, pages 1415--1424.

\bibitem[{Yingzhen and Mandt(2018)}]{yingzhen2018disentangled}
Li~Yingzhen and Stephan Mandt. 2018.
\newblock Disentangled sequential autoencoder.
\newblock In \emph{International Conference on Machine Learning}, pages
  5656--5665.

\bibitem[{Zhang et~al.(2020)Zhang, Chen, Gan, Wang, Shen, Wang, Wen, and
  Carin}]{ruiyi2020improving}
Ruiyi Zhang, Changyou Chen, Zhe Gan, Wenlin Wang, Dinghan Shen, Guoyin Wang,
  Zheng Wen, and Lawrence Carin. 2020.
\newblock Improving adversarial text generation by modeling the distant future.
\newblock \emph{Proceedings of the 58th Annual Meeting of the Association for
  Computational Linguistics}.

\bibitem[{Zhang et~al.(2019)Zhang, Yu, Shen, Jin, and Chen}]{zhang2019text}
Ruiyi Zhang, Tong Yu, Yilin Shen, Hongxia Jin, and Changyou Chen. 2019.
\newblock Text-based interactive recommendation via constraint-augmented
  reinforcement learning.
\newblock In \emph{Advances in neural information processing systems}, pages
  15214--15224.

\bibitem[{Zhao et~al.(2018)Zhao, Kim, Zhang, Rush, and
  LeCun}]{zhao2017adversarially}
Junbo Zhao, Yoon Kim, Kelly Zhang, Alexander Rush, and Yann LeCun. 2018.
\newblock Adversarially regularized autoencoders.
\newblock In \emph{Proceedings of the 35th International Conference on Machine
  Learning}, pages 5902--5911.

\bibitem[{Zhou et~al.(2019)Zhou, Liu, Liu, Luo, and Wang}]{zhou2019talking}
Hang Zhou, Yu~Liu, Ziwei Liu, Ping Luo, and Xiaogang Wang. 2019.
\newblock Talking face generation by adversarially disentangled audio-visual
  representation.
\newblock In \emph{Proceedings of the AAAI Conference on Artificial
  Intelligence}, volume~33, pages 9299--9306.

\end{thebibliography}
\bibliographystyle{acl_natbib}

\newpage
\appendix
\onecolumn
\section{Proofs of Theorems}\label{sec:supplemental}

\begin{proof}[Proof of Theorem~\ref{thm:upper-bound}]
First, we show that 
\begin{equation}\label{eq:club}
    \bbE_{p(\vs,\vc)} [\log p(\vs|\vc)] - \bbE_{p(\vs)p(\vc)}[\log p(\vs|\vc)] \geq \MI(\vs;\vc).
\end{equation}
Calculate the gap $\Delta$ between the left-hand  side and right-hand side of  Eq.~\eqref{eq:club}:
\begin{align*}
    \Delta =& \bbE_{p(\vs,\vc)} [\log p(\vs|\vc)] - \bbE_{p(\vs)p(\vc)}[\log p(\vs|\vc)] - \MI(\vs;\vc)\\
    =& \bbE_{p(\vs,\vc)} [\log p(\vs|\vc)] - \bbE_{p(\vs)p(\vc)}[\log p(\vs|\vc)] - \bbE_{p(\vs,\vc)}\left[\log p(\vs|\vc) - \log p(\vs)\right] \\
    =& \bbE_{p(\vs,\vc)}[\log p(\vs)] - \bbE_{p(\vs)}\bbE_{p(\vc)} [\log p(\vs|\vc)]\\
    =& \bbE_{p(\vs)}\left[\log p(\vs) - \bbE_{p(\vc)}[\log(p(\vs|\vc)]\right] \\
    = & \bbE_{p(\vs)} \left[\log\left(\bbE_{p(\vc)}[p(\vs|\vc)] \right) - \bbE_{p(\vc)}[\log p(\vs|\vc)] \right] \geq 0. & \text{(Jensen's Inequality)}
\end{align*}
Therefore, the inequality in Eq.~\eqref{eq:club} holds.

Given sample pairs $\{(\vs_j, \vc_j)\}_{j=1}^M \sim p(\vs, \vc)$,  the left-hand side of Eq.~\eqref{eq:club} has an unbiased estimation:
\begin{align*}
&\frac{1}{M} \sum_{j=1}^M \bbE_{(\vs_j,\vc_j) \sim p(\vs,\vc)}\left[\log p(\vs_j| \vc_j) \right] - \frac{1}{M^2}\sum_{j=1}^M \sum_{k=1}^M \bbE_{\vs_j \sim p(\vs)} \bbE_{\vc_k \sim p(\vc)} \left[\log p(\vs_j|\vc_k)\right] \\
= &\bbE\left[ \frac{1}{M}\sum_{j=1}^M \left[\log p(\vs_j|\vc_j) - \frac{1}{M}\sum_{k=1}^M \log p(\vs_j | \vc_k) \right]\right] = \bbE\left[\frac{1}{M} \sum_{j=1}^M R_j\right],
\end{align*}
which is what we claim in Theorem 3.1.
\end{proof}

\begin{proof}[Proof of Lower Bounds in Eq.~\eqref{eq:lower-bound-in-obj}]

\begin{align*}
    \MI(\vc; \vx) &= \mathbb{E}_{p(\vx, \vc)} [\log p(\vx | \vc) - \log p(\vx)] 
    = \Enpy(\vx) + \mathbb{E}_{p(\vx, \vc)} [\log p(\vx | \vc)] \\
    & = \Enpy(\vx) + \bbE_{p(\vx,\vc)} [\log p(\vx|\vc) - \log q_\phi(\vx|\vc) + \log q_\phi(\vx|\vc)] \\
    & = \Enpy(\vx) + \bbE_{p(\vx,\vc)}[\log p(\vx| \vc) - \log q_\phi(\vx| \vc)] + \bbE_{p(\vx,\vc)}[\log q_\phi(\vx|\vc)] \\
    & = \Enpy(\vx) + \KL(p(\vx|\vc) \Vert q_\phi(\vx| \vc)) + \bbE_{p(\vx,\vc)}[\log q_\phi(\vx|\vc)] \\
& \geq H(\vx) + \mathbb{E}_{p(\vx, \vc)} [\log q(\vx | \vc)].
\end{align*}
The inequality is based on the fact that the KL-divergence is always non-negative.
The lower bound for $\MI(\vs;y)$ can be also derived in the similar way.
\end{proof}

\section{Detailed Experimental Setups}
 We set the dimension of style embedding to be smaller than the content embedding, because the content carries more information than the style of sentences. The hyper-parameter $\beta$ in our loss function is a formal expression of re-weighting the two objectives of disentanglement and autoencoding. In practice, we vary it from 0 to 1 with step 0.1 during the first 10 training epochs. At the beginning of the training, the output latent embeddings are not representative enough. Therefore, we choose a small weight on the disentanglement  term to avoid obstructing the learning of representative embeddings. After the latent embedding is sufficiently trained, which can successfully reconstruct the input sentences, we slowly enlarge $\beta$ for the disentanglement. After $\beta$ reaches 1, we fix it until all the training epochs are finished. 

\section{Sample-based Embedding Divergences}

In this section we introduce the implementation details of the calculation about label-embedding correlation. 
As mentioned in Section~5.4 , the distribution divergence between $p(\vc | y=0)$ and $p(\vc | y =1)$ measures the correlation between content embeddings and style labels.  Assume $\vc_1^{(0)}, \vc_2^{(0)}, \dots, \vc_{N_0}^{(0)} \sim p(\vc| y = 0)$, and $ \vc_1^{(1)}, \vc_2^{(1)}, \dots, \vc_{N_1}^{(1)}\sim p(\vc| y = 1)$, then the four metrics MAD, ED, WD, MMD are calculated based on the two groups of samples.
  With a ground distance $d(\cdot, \cdot)$, the implementaion of the above four metrics are demonstrated in following:
 \begin{equation}
    D_\text{MAD} = d(\frac{1}{N_0} \sum_{i=1}^{N_0} \vc_{i}^{(0)},
    \frac{1}{N_1}  \sum_{j=1}^{N_1} \vc_{j}^{(1)}).
  \end{equation}

\begin{equation}
    D_\text{ED} =   \frac{2}{N_0 N_1} \sum_{i=1}^{N_0} \sum_{j=1}^{N_1} d(\vc_{i}^{(0)}, \vc_{j}^{(1)}) - \frac{1}{N_0^2} \sum_{i=1}^{N_0} \sum_{j=1}^{N_0} d(\vc_{i}^{(0)}, \vc_{j}^{(0)}) -  \frac{1}{N_1^2} \sum_{i=1}^{N_1} \sum_{j=1}^{N_1} d(\vc_{i}^{(1)}, \vc_{j}^{(1)})
    \end{equation}
    
 \begin{equation}
       D_\text{WD} = \min_{p_{ij}} \sum_{i=1}^{N_0} \sum_{j=1}^{N_1} p_{ij} \ d(\vc_{i}^{(0)}, \vc_{j}^{(1)})  \ \ \ 
      s.t.   \sum_{i=1}^{N_0} p_{ij} = \frac{1}{N_1}, \ \ \ \sum_{j=1}^{N_1} p_{ij} = \frac{1}{N_0}. 
\end{equation}

\begin{equation}
    D_\text{MMD} = \frac{1}{N_0^2} \sum_{i=1}^{N_0} \sum_{j=1}^{N_0} K(\vc_{i}^{(0)}, \vc_{j}^{(0)}) +  \frac{1}{N_1^2} \sum_{i=1}^{N_1} \sum_{j=1}^{N_1} K(\vc_{i}^{(1)}, \vc_{j}^{(1)}) -  \frac{2}{N_0 N_1} \sum_{i=1}^{N_0} \sum_{j=1}^{N_1} K(\vc_{i}^{(0)}, \vc_{j}^{(1)}),
    \end{equation}
    where $K(\cdot, \cdot)$ is a kernel function. Here we choose $K(\cdot, \cdot)$ from RBF kernel family with bandwidth $w= 1$. 
    
   For style embedding,  the calculation formats are the same as in above equations. The style embeddings and content embeddings have different dimensions, which leads to the ground metric $d(\cdot, \cdot)$  inconsistent. Therefore, instead of using Euclidean distance, we use the cosine distance as the ground metric.
\section{Details in Representation Quality Evaluation}
For style preservation, we pretrain a style classifier on each dataset. The style classifier is built by a one-layer LSTM appended with a multi-head attention layer. The number of the attention head is set to 6. The classifiers reach 95\% prediction accuracy on Yelp and 93\% prediction accuracy on Personality-Captioning. We input transferred sentences into the classifier and test whether the predicted style label is the same as the target style label.

For human evaluation, we transferred 1000 sentences with randomly selected style labels. After the transferring, we ask 10  human annotators to justify the style label, content preservation and content fluency. The style label is 0 or 1 representing the positive or negative sentiment of the given sentence. The content preservation and the content fluency is scored between 0 to 5. To make the style accuracy compatible with the other two scores, we scale it into range [0,5].  If the scores from the two annotators have a difference larger than 2, the scores will not be recorded. In this way, we ensure the evaluation criteria of annotators are similar.
\end{document}